\documentclass{amsart}

\usepackage{latexsym,amsfonts,amsthm,amssymb,setspace}
\usepackage{MnSymbol}
\usepackage{mathrsfs}
\usepackage{marvosym}
\usepackage{subfigure}
\usepackage{paralist}
\usepackage{algorithm}
\usepackage{algorithmic}
\usepackage{multirow}
\usepackage{setspace}
\usepackage{amssymb}

\usepackage{marvosym}
\usepackage{paralist}
\usepackage{color}

\usepackage{graphicx}
\usepackage{fontenc}%

\usepackage[verbose]{wrapfig}

\usepackage{subfigure}
\renewcommand{\thesubfigure}{\thefigure.\arabic{subfigure}}
\makeatletter
\renewcommand{\p@subfigure}{}
\renewcommand{\@thesubfigure}{\thesubfigure:\hskip\subfiglabelskip}
\makeatother

\newtheorem{theorem}{Theorem}
\newtheorem{definition}{Definition}
\newtheorem{lemma}{Lemma}

\newcommand{\abs}[1]{\lvert#1\rvert}

\graphicspath{{Images/}}
\DeclareGraphicsExtensions{.eps}

\begin{document}

\title[Vorono\"{i} Region based Image Segmentation]{Vorono\"{i} Region-Based Adaptive Unsupervised Color Image Segmentation}

\author[R. Hettiarachchi]{R. Hettiarachchi$^{\alpha}$}
\email{James.Peters3@umanitoba.ca}
\address{\llap{$^{\alpha}$\,}Computational Intelligence Laboratory,
University of Manitoba, WPG, MB, R3T 5V6, Canada
}
\author[J.F. Peters]{J.F. Peters$^{\alpha, \beta}$}
\address{\llap{$^{\beta}$\,} Department of Mathematics, Faculty of Arts and Sciences, Ad\.{i}yaman University, 02040 Ad\.{i}yaman, Turkey}


\date{}

\begin{abstract}
Color image segmentation is a crucial step in many computer vision and pattern recognition applications. This article introduces an adaptive and unsupervised clustering approach based on Vorono\"{i} regions, which can be applied to solve the color image segmentation problem. The proposed method performs region splitting and merging within Vorono\"{i} regions of the Dirichlet Tessellated image (also called a Vorono\"{i} diagram) , which improves the efficiency and the accuracy of the number of clusters and cluster centroids estimation process. Furthermore, the proposed method uses cluster centroid proximity to merge proximal clusters in order to find the final number of clusters and cluster centroids. In contrast to the existing adaptive unsupervised cluster-based image segmentation algorithms, the proposed method uses K-means clustering algorithm in place of the Fuzzy C-means algorithm to find the final segmented image. The proposed method was evaluated on three different unsupervised image segmentation evaluation benchmarks and its results were compared with two other adaptive unsupervised cluster-based image segmentation algorithms. The experimental results reported in this article confirm that the proposed method outperforms the existing algorithms in terms of the quality of image segmentation results. Also, the proposed method results in the lowest average execution time per image compared to the existing methods reported in this article.
\end{abstract}


\maketitle
\section{Introduction} \label{sec:intro}

Image segmentation plays a major role in many computer vision and pattern recognition applications. According to \cite{cheng2001color}, image segmentation is the process of dividing an image into different regions such that each region is, but the union of any two adjacent regions is not, homogenous. The existing image segmentation techniques can be broadly categorized into threshold-based, clustering-based, region-based, edge-based and physics-based segmentation approaches \cite{ladys1994colour,cheng2001color}. There are various hybrid image segmentation techniques, which combine two or more of the aforementioned approaches. Clustering techniques have been widely used to cluster image pixels based on the similarity of their features (e.g. color, texture, etc.). K-means \cite{hartigan1979algorithm} and Fuzzy C-means (FCM) \cite{cannon1986efficient} are two of the most popular clustering techniques used for image segmentation.

Clustering is an unsupervised learning process, which does not require class labeled data set as training data to cluster unknown set of data into clusters. According to \cite{anil1998clustering}, a cluster is comprised of a number of similar objects (pixels in our case) collected or grouped together. In traditional clustering techniques such as K-means and FCM, the number of clusters has to be predefined to initiate the algorithm. Also, the initial cluster centers (or centroids) are not known. For example, the K-means algorithm starts with random cluster centers initially as starting points and then iteratively adjusts the centroids until the algorithm converges. This process is time consuming and may not produce the intended result. 

In contrast to K-means, FCM returns the membership weights for each point in the data set, which define the degree of belonging of a data point to a given cluster. A given data point will have a high membership weight (higher degree of belonging) to a nearby cluster a low membership weight for a faraway cluster. \cite{ghosh2013comparative} claims that K-means is faster and better compared to FCM algorithm.

In the case of color images, which are complex data sets by nature, the determination of the number of pixel clusters and the cluster centroids becomes very challenging. Thus, adaptive unsupervised clustering techniques, which automatically find the number of clusters and the corresponding cluster centroids is vital for successful image segmentation via clustering techniques. Due to this reason, during the recent past, the focus of many researchers has turned towards adaptive unsupervised clustering techniques for image segmentation. Mean-Shift algorithm introduced by Fukunaga and Hostetler \cite{fukunaga1975estimation} is a non-parametric clustering algorithm, which does not require the number of clusters to be predefined. However, Mean-Shift is very much slower compared to K-means algorithm. The time complexity of classical K-means is $O(knT)$ while the time complexity of Mean-Shift is $O(Tn^{2})$.

\begin{figure}[!htb]
	\centering
		\includegraphics[width=0.9\textwidth]{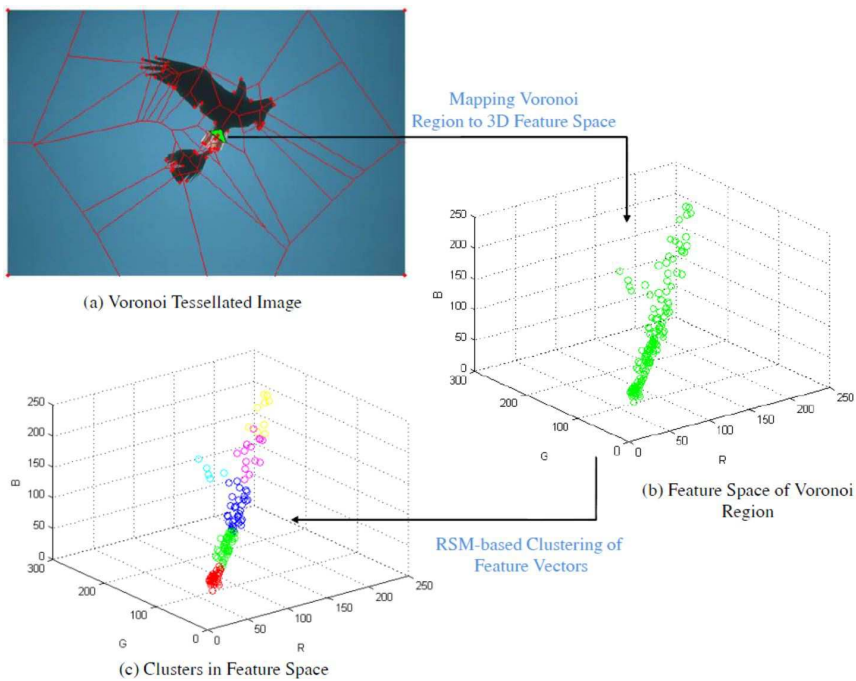}
	\caption{Intra-Vorono\"{i} region clustering}
	\label{fig:fig1}
\end{figure}

Recently, Yu et al. in \cite{yu2010adaptive} proposed an adaptive unsupervised algorithm called Ant Colony–Fuzzy C-means Hybrid Algorithm (AFHA), which is a combination of Ant System and Fuzzy C-mean techniques. In \cite{yu2010adaptive}, Ant System (AS) \cite{colorni1991distributed} is used to determine the number of clusters and the cluster centroids. It is said in \cite{yu2010adaptive} that AFHA outperforms other state-of-the-art segmentation technique, such as X-means \cite{pelleg2000x}, Mean-Shift, Normalized Cut \cite{shi2000normalized}. Another unsupervised adaptive clustering approach for image segmentation named Region Splitting and Merging-Fuzzy C-means Hybrid Algorithm (RFHA) is proposed in \cite{tan2013color} by Tan et al. RFHA algorithm uses region splitting and merging scheme to determine the number of clusters and cluster centroids. Both of these algorithms use Fuzzy C-means algorithm to find the final segmented image. Some other adaptive clustering approaches for image segmentation can be found in \cite{vantaram2011adaptive,sujaritha2010color,rosenberger2000unsupervised,ilea2008ctex,ugarriza2009automatic,bhoyar2010color}.

The computational complexity of the AFHA algorithm is very high due to the complicated nature of the AS module that it uses. Thus Yu et al. proposed a modified version of AFHA algorithm named improved AFHA (IAFHA)in \cite{yu2010adaptive}, which finds the number of clusters and cluster centroids via AS by taking only a small proportion (about 30\%) of the total number of pixels into account. This modification in IAFHA significantly reduces the computational complexity of the original AFHA algorithm, but it affects the performance of the algorithm at the same time. RFHA is faster compared to AFHA. However, both of these algorithms suffer from the high computational complexity of the FCM algorithm. 

In this article, we propose a Vorono\"{i} region-based adaptive unsupervised algorithm to automatically find the number of clusters and the cluster centroids in a given set of pixels. First, the proposed algorithm adaptively divides the image into Vorono\"{i} regions and then automatically finds the number of clusters and cluster centroids of the pixels belonging to each of these Vorono\"{i} regions by using a region splitting and merging scheme similar to the one given in \cite[\S 2.1]{tan2013color} (see figure \ref{fig:fig1}). Next, the intra-Vorono\"{i} region clusters that are near each other will be merged together to find the final number of clusters and cluster centroids in a given image. Finally, the final segmented image will be found by applying K-means clustering algorithm on the whole image initiated with the number of clusters and cluster centroids found in the previous step.


The Vorono\"{i} region wise clustering in the proposed algorithm reduces the complexity of the segmentation problem significantly, which will be discussed in detail in section \ref{sec:compcomplex}. Furthermore, since the number of possible clusters within a single Vorono\"{i} region is usually lower compared to the number of clusters in the whole image, estimating the number of clusters and cluster centroids becomes more efficient and precise.   

The rest of the article is organized as follows. Section \ref{sec:relwork} presents the work related to the method proposed in this article. The methodology of the propsoed Vorono\"{i} region-based adaptive unsupervised algorithm is presented in section \ref{sec:method}. Section \ref{sec:exp} discusses the results of experiments conducted on the proposed method and two other adaptive unsupervised cluster-based image segmentation algorithms. Finally, section \ref{sec:conc} concludes the research work presented in this article. 

\section{Related Work} \label{sec:relwork}

\subsection{Image Segmentation}\label{sec:imgseg}

A formal definition of image segmentation given in \cite{pal1993review} is as follows.

\begin{definition}\label{def:def1}\textbf{Image Segmentation}\cite{pal1993review}\\
If $F$ is the set of all pixels and $P()$ is a uniformity (homogeneity) predicate defined on group of connected pixels, then segmentation is a partitioning of the set $F$ into a set of connected subsets of regions $(S_1,S_2,\ldots,S_n)$ such that
\[
\bigcup_{i=1}^n S_i=F \mbox{ with } S_i \cap S_j =\emptyset, i\neq j
\]
\end{definition}

The uniformity predicate $P(S_i)=true$ for all regions $(S_i)$ and $P(S_i \cap S_j)=false$, when $S_i$ is adjacent to $S_j$. Based on this definition, a segmented image can be evaluated by measuring the homogeneity within each segment and by measuring the overlap between each pair of segments. Haralick et al. proposed four criteria to define a good image segmentation in \cite{haralick1985image} as follows.

\begin{enumerate}
	\item Regions should be uniform and homogeneous with respect to some characteristic(s).
	\item Adjacent regions should have significant differences with respect to the characteristic on which they are uniform.
	\item Region interiors should be simple and without holes.
	\item Boundaries should be simple, not ragged, and be spatially accurate.
\end{enumerate}

The first two criteria can be directly derived from the definition \ref{def:def1}. Most of the segmentation evaluation methods are based on the first two criteria defined by Haralick et al. \cite{haralick1985image} given in section \ref{sec:imgseg}, jointly called the characteristic criteria. The first criterion measures the intra-region uniformity while the second criterion measures the inter-region disparity. Zhang et al. \cite{zhang2008image} provides a good summary of unsupervised evaluation methods, which fall under both of these categories. 

\subsection{Dirichlet Tessellation (Vorono\"{i} Diagram)} \label{ss:DT}
Image tessellation is a tiling of an image surface with regular polygons and Dirichlet tessellation (also called Vorono\"{i} diagram) is one example of image tessellation. Dirichlet~\cite{Dirichlet1850} introduced polygon-based tessellation in 1850, which was elaborated by Vorono\"{i} in 1907~\cite{Voronoi1907}. A Vorono\"{i} diagram is the partitioning of a plane with $n$ points into convex polygons such that each polygon contains exactly one generating point and every point in a given polygon is closer to its generating point than to any other. Thus, a Vorono\"{i} region can be defined by the following definition,

\begin{definition}\label{def:def2}\textbf{Vorono\"{i} Region}\\
Let $S$ and $X$ be finite sets in a $n$-dimensional Euclidean space. Let $p \in S$ (denoted $V_p$) is defined by
\[
V_p= \left\{ x \in X: \left\| x-p \right\| \leq_{\forall q \in S} \left\| x-q \right\| \right\}.
\]
where $S$ is the set of generating points.
\end{definition}

Thus, in order to generate the Vorono\"{i} diagram, the generating points or the seed points have to be provided. In our context, these seed points can be \emph{corners, centroids, critical points, key points} found in images. For example, Du et al. introduced the technique of Centroidal Vorono\"{i} Tessellations (CVT) in 1999 \cite{du1999centroidal}, which uses centroids as the generating points for the Vorono\"{i} tessellation. In 2006, Du et al. revisits the CVT algorithm in their subsequent article \cite{du2006centroidal}, which focus more on the applications of CVT.

Vorono\"{i} regions (convex polygons produced by Dirichlet tessellation) have been explored as a solution to the image segmentation problem during the past two decades. An interesting paper on Vorono\"{i} based image segmentation is \emph{On Points Geometry for Fast Digital Image Segmentation} \cite{cheddad2008points}. In this paper, rather than applying the Vorono\"{i} Diagram on the image itself, it is applied on a few selected points by using the image histogram. \cite{suhail2002content} also suggests Vorono\"{i} cells for image segmentation. The articles \cite{du1999centroidal,du2006centroidal} focus on spatial Dirichlet tessellation while \cite{cheddad2008points} focus on Dirichlet tessellation on image histograms. 

\subsection{Region Splitting and Merging Scheme} 
A summary of the region splitting and merging scheme proposed in \cite[\S 2.1]{tan2013color} is given below.

\begin{enumerate}
	\item Apply the moving average filter with a span of 5 to the histogram of each color channel.
	\item Identify and remove all the local peaks and valleys in the histogram of each color channel based on the fuzzy rule base given in equation (2) and (3) in \cite[\S 2.1]{tan2013color}.
	\item Identify the significant peaks by examining the turning points from positive to negative gradient changes.
	\item Determine the valleys of each color channel histogram by taking the minimum value between any adjacent peaks.
	\item Define color cells (classes) based on the combinations of valleys in the histogram of each color channel. 
	\item Assign each pixel to a color cell depending on its color channel values.
	\item Calculate initial cluster centers by averaging pixels withing a given cell.
	\item Calculate Manhattan distances between all pairs of cluster centers and find the two nearest clusters.
	\item Merge the two nearest clusters and update the cluster center.
	\item Reduce the number of clusters and repeat the process until the shortest distance between two nearest clusters is not less than a predefined threshold.
\end{enumerate}

\section{Proposed Method} \label{sec:method}
As explained in section \ref{ss:DT}, Dirichlet tessellation divides an image into polygonal regions based on the spatial locations (X and Y coordinates) of the seed points (generating points). These polygonal Vorono\"{i} regions are much simpler to analyze and process compared to the whole image. Also, since the seed points can be derived from a given image, a tessellation is tailored to the structure of the image itself. Since, these image features are found spatially in an image, we call this process \emph{spatial Dirichlet tessellation}, which divides an image into $N$ number of Vorono\"{i} regions, where $N$ is the number of seed points.

\[
	V=\left\{ V_1, V_2,\ldots,V_N \right\}.
\]

Once a Dirichlet tessellation of a given image is found, the next step is to further divide each Vorono\"{i} region $V_p$ until the whole image is grouped into small similar regions. Thus, next we consider the feature space of each Vorono\"{i} region. If each pixel within a given Vorono\"{i} region $V_p$ is represented by its d-dimensional feature vectors, then each pixel in $V_p$ can be mapped to a point on a d-dimensional feature space. 

\begin{figure}[!htb]
	\centering
		\includegraphics[width=1.00\textwidth]{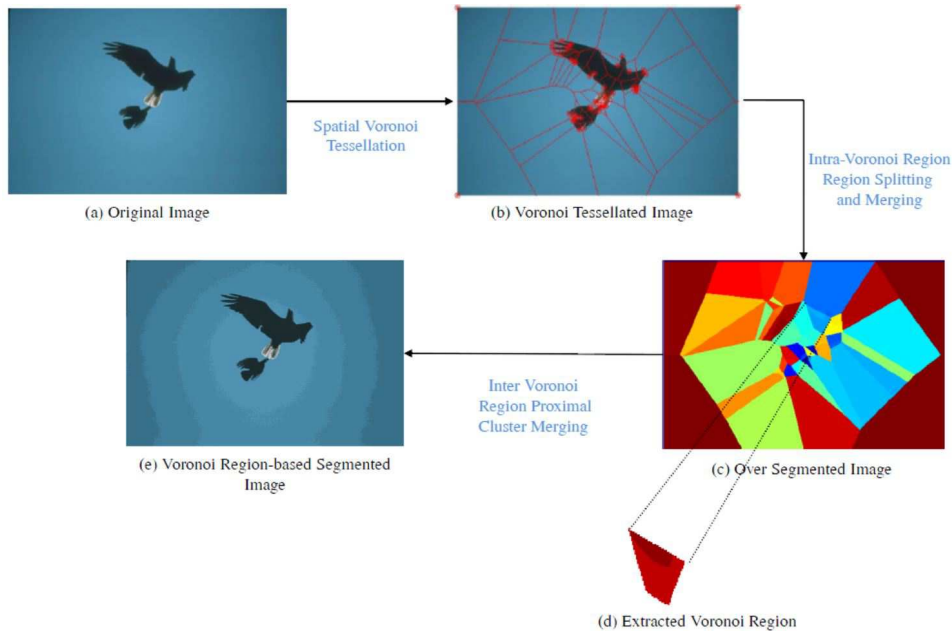}
	\caption{Stages of the proposed method}
	\label{fig:fig2}
\end{figure}

\begin{equation}\label{eq:3}
V_p \rightarrow X_p, \text{ where } X_p \subset \mathbb{R}^d \text{ and } \Phi: \mathbb{R}^2 \rightarrow \mathbb{R}^d \text{ defined by } \Phi(x)=\left\{ \Phi_1(x), \Phi_2(x), \ldots \Phi_d(x) \right\}=\overline{x}.
\end{equation}
where $\Phi$ is a set of probe functions representing features of a given pixel $x$. $\Phi(x) = \overline{x}$ is the feature vector representing the pixel $x$.

Then, Dirichlet tessellation can be performed on the feature space $\mathbb{R}^d$, where $d=3$ in this case as we use R, G and B color channel values of each pixel. For this purpose, the Lloyd style K-means clustering algorithm \cite{ostrovsky2006effectiveness} can be used. K-means clustering partitions the space of observations into $k$ classes such that each observation belongs to the cluster with the nearest mean or the centroid. In an optimal solution, each centroid is assigned the data in its Vorono\"{i} region and is located at the center of mass of this data. The Lloyd style K-means clustering algorithm runs iteratively until all the data points in the Vorono\"{i} region of a given centroid are clustered together, and then the centroid is moved to the centroid of its cluster \cite{ostrovsky2006effectiveness}. The original Lloyd's algorithm \cite{lloyd1982least}, which is also known as Vorono\"{i} iteration differs from k-means clustering in that its input is a continuous geometric region rather than a discrete set of points as in K-means clustering.

Once the k-means clustering is performed, the discrete set of points in the feature space fall within the Vorono\"{i} regions of the centroids closest to those regions. Therefore, K-means clustering can be used to cluster data into Vorono\"{i} regions based on their features. The K-means clustering algorithm requires the number of clusters ($k$) to be predefined. In the case of the whole image, automatically finding the number of clusters is computationally complex. However, in the case of small Vorono\"{i} region, finding the number of clusters will be more effective due to the small number of pixels and small number of clusters within a given Vorono\"{i} region. Thus, we propose to use the region splitting and merging (RSM) scheme given in \cite{tan2013color} to adaptively determine the number of clusters and cluster centroids for each Vorono\"{i} region.

First, the set of pixels within $V_p$ will be mapped into its feature space $X_p$ with the help of equation (\ref{eq:3}). Then, region splitting and merging will be applied on $X_p$ in order to find the adaptive feature space Dirichlet tessellation. Let $k_p$ be the number of clusters and $C_p$ be the cluster centroids found by applying RSM on the set of pixel feature points $X_p$ within a given Vorono\"{i} region $V_p$, then $X_p$ can be represented as a set of clusters as given below. Figure \ref{fig:fig1} depicts the process of intra-Vorono\"{i} region clustering.


\[
X_p= \{s_1,s_2,\ldots,s_{k_p}\}. 
\]

Once, the feature space Dirichlet tessellation is completed for all the Vorono\"{i} regions in the image, the image can said to be at an over-segmented state. Let the total set of clusters after the feature space Dirichlet tessellation be $S = \{ s_1,s_2,\ldots,s_P\}$. Since, the clustering was performed within each Vorono\"{i} region, there can be clusters with similar features belonging to two different Vorono\"{i} regions. Thus, in the next step we find such similar clusters and merge them together to find the final number of clusters and cluster centroids.  

In the proposed method, the concept of nearness (proximity) (\cite[\S 1.19]{peters2014topology}, \cite[\S 1.4]{peters2016proximity} and~\cite{PetersInan2016arXiv}) will be used to find clusters, which are very much similar (proximal) to each other. These proximal clusters can be merged together to form a single cluster. We extend the notion of nearness of clusters to nearness of their respective centroids to compare two clusters.

\begin{definition}\label{def:def3} \textbf{Cluster Centroid}\\
Let $s$ be cluster, then the centroid of the cluster $\mu$ can be defined as,
\[
\overline{\mu}=\frac{\sum_{\forall \overline{x_i} \in s}{\overline{x_i}}}{n}
\]

where $\overline{x_i}$ are the feature vectors representing the members of $s$ and $n$ is the number of members in $s$.
\end{definition}


The \emph{Centroid} method can be used, when finding inter-Vorono\"{i} region proximal clusters. In the centroid method, \emph{the resemblance between two clusters is equal to the resemblance between their centroids} \cite[\S 9.5]{romesburg2004cluster}. Thus, clusters whose centroids are closest together are merged. The centroid of a merged cluster is a weighted combination of the centroids of the two individual clusters, where the weights are proportional to the sizes of the clusters \cite[pp. 373]{noruvsis2012ibm}. 

Let $s_1, s_2$ be two clusters and $\overline{\mu_1},\overline{\mu_2}$ be there cluster centroids respectively. Then the lemma \ref{lem:lem1} can be derived for the Centroid method,

\begin{lemma}\label{lem:lem1}
Let $s_i, s_j$ be two clusters in $S$ and $\overline{\mu_i},\overline{\mu_j}$ be their cluster centroids respectively. Then,
\[
proximity(s_i,s_j)=proximity(\overline{\mu_i},\overline{\mu_j}).
\]
\end{lemma}
\begin{proof}
\[
\begin{array}{l l}
proximity(\overline{\mu_i},\overline{\mu_j}) & =\overline{\mu_i} \cdot \overline{\mu_j} \\
 & =\left( \frac{1}{n_i} \sum_{\overline{x_i} \in s_i} \overline{x_i} \right) \cdot \left( \frac{1}{n_j} \sum_{\overline{x_j} \in s_j} \overline{x_j} \right)\\
 & =\frac{1}{n_i n_j} \sum_{\overline{x_i} \in s_i} \sum_{\overline{x_j} \in s_j} \overline{x_i} \cdot x_j\\
 & = proximity(s_i,s_j) 
\end{array}
\]
\end{proof}

In the context of this article, the centroid proximity will be defined as,

\begin{definition} \label{def:def4} \textbf{Centroid Proximity} \\
Let $s_i, s_j$ be two clusters in $S$ and $\overline{\mu_i},\overline{\mu_j}$ be their cluster centroids respectively. The centroid proximity between $\overline{\mu_i}$ and $\overline{\mu_j}$ will be defined as,
\[
proximity(\overline{\mu_i},\overline{\mu_j})=d(\overline{\mu_i},\overline{\mu_j})
\]
where, $d(\overline{\mu_i},\overline{\mu_j})$ is the Manhattan distance (proximity measure) between the two centroids $\overline{\mu_i}$ and $\overline{\mu_j}$ given by, 
\[
d(\overline{\mu_i},\overline{\mu_j}) = \mathop{\sum}\limits_{i=1}^n \abs{\Phi_i(\mu_i)-\Phi_i(\mu_j)}.
\]
\end{definition}

\begin{theorem}\label{th:th1}
Let $s_i, s_j$ be two clusters in $S$ and $\mu_i,\mu_j$ be their cluster centroids respectively.
\[
If \ d(\overline{\mu_i},\overline{\mu_j}) < \varepsilon, \text{ then } s_i \ \delta_d \ s_j 
\]
where $d$ is the Manhattan distance metric leading to metric proximity $\delta_d$ and , $\varepsilon$ is a predefined threshold. 
\end{theorem}

\begin{proof}
From lemma \ref{lem:lem1} and definition \ref{def:def4}, $proximity(s_i,s_j)=proximity(\overline{\mu_i},\overline{\mu_j})=d(\overline{\mu_i},\overline{\mu_j})$. Thus, if $d(\overline{\mu_i},\overline{\mu_j}) < \varepsilon$, then $\overline{\mu_i}\ \delta_d \ \overline{\mu_j} \Rightarrow s_i \ \delta_d \ s_j$.
\end{proof}

\begin{definition} \label{def:def5} \textbf{Proximal Clusters} \\
Let $s_i, s_j$ be two clusters in $S$ and $\overline{\mu_i},\overline{\mu_j}$ be their cluster centroids respectively. The clusters $s_i$ and $s_j$ are called proximal clusters (denoted by $s_i \ \delta_d \ s_j$) if, $d(\overline{\mu_i},\overline{\mu_j}) < \varepsilon$, where $\varepsilon$ is a predefined threshold.
\end{definition}

Thus, two clusters $s_i,s_j$ can be merged if $d(\overline{\mu_i},\overline{\mu_j}) < \varepsilon$ and the centroid of the merged cluster can be calculated by using equation (\ref{eq:eq2}),

\begin{equation}\label{eq:eq2}
	\overline{\mu} = \frac{n_i}{n_i+n_j}\overline{\mu_i} + \frac{n_j}{n_i+n_j}\overline{\mu_j}.
\end{equation}

where $n_i$ and $n_j$ are the number of elements in $s_i$ and $s_j$ respectively. 

In the proposed method, the following merge condition derived based on theorem \ref{th:th1} will be used to iteratively merge proximal clusters belonging to different Vorono\"{i} regions.

\textbf{Merge Condition :}
Let $s_i, s_j$ be two clusters in $S$ such that $s_i \subset X_i, s_j \subset X_j, \text{ where } V_i \rightarrow X_i \text{ and } V_j \rightarrow X_j, V_i \neq Vj$. $s_i$ and $s_j$ will be merged if $s_i \ \delta_d \ s_j$.

All proximal clusters, which satisfy the above merge condition will be found and merged together iteratively until the number of proximal clusters becomes zero. This process yields the final number of clusters $k$ and the cluster centroids $C$. Finally, K-means clustering on the feature vectors of the total set of pixels $X=\bigcup_{p=1}^N X_p$ of the image will be performed initiated with $k$ and $C$ to find the final segmented image. 

\subsection{Implementation of the Proposed Algorithm}
In the implementation of the proposed algorithm, corner points of an image will be used as generating points of the spatial Vorono\"{i} tessellation. The proposed algorithm consists of the following major steps:

\begin{enumerate}
	\item Find the (X,Y) coordinates of the set of the most prominent corners in the image $I$.
	\item Generate the Vorono\"{i} diagram $V=\left\{ V_1, V_2,\ldots,V_N \right\}$ by taking the set of corner points found in step 1 as the generating points.
	\item Within each Voron\"{i} region $V_p$,
	\begin{enumerate}
		\item Apply RSM on the feature vector set $X_p$ of each Vorono\"{i} region $V_p$ to find the number of clusters $k_p$, cluster centroids $C_p$ and the set of clusters $X_p= \{s_1,s_2,\ldots,s_{k_p}\}$.
		\item Let the total set of clusters belonging to all Vorono\"{i} regions be $S = \{ s_1,s_2,\ldots,s_P\}$.
	\end{enumerate}
	\item Iteratively merge the inter Vorono\"{i} region proximal clusters until the number of proximal clusters become zero.
	\begin{enumerate}
		\item Find pair of clusters $s_i, s_j\subset S$, which satisfy the merge condition and $d(\overline{\mu_i},\overline{\mu_j})=inf_{\forall s_x,s_y \subset S}\left( d(\overline{\mu_x},\overline{\mu_y}) \right)$. where $\overline{\mu_i}$ and $\overline{\mu_j}$ are the centroids of clusters $s_i$ and $s_j$ respectively.
		\item Merge $s_i$ and $s_j$ together so that $s_k=s_i \cup s_j$.
		\item Find new cluster centroid $\mu_k$ of cluster $s_k$ using the equation (\ref{eq:eq2}).
		\item Update the total number of clusters $P=P-1$.
		\item Repeat steps 4.(a) to 4.(d) until $d(\overline{\mu_i},\overline{\mu_j}) < \varepsilon$, where $\varepsilon$ is the proximity threshold.
	\end{enumerate}	
\end{enumerate}

The pseudo code of the Vorono\"{i}-based segmentation algorithm is given in algorithms \ref{algo:voronoi} and the stages of the proposed method are depicted in figure \ref{fig:fig2}.

\begin{algorithm}[!htb]
\begin{algorithmic}[1] \label{algo:voronoi}
\renewcommand{\algorithmicrequire}{\textbf{Input:}}
\renewcommand{\algorithmicensure}{\textbf{Output:}}
\REQUIRE Image $I$, Proximity Threshold $\varepsilon$
\ENSURE  Vorono\"{i} Segmented Image $I_{seg}$  
\STATE \hspace{-0.9em}\textbf{function} Vorono\"{i}\_segmentation ($I,\varepsilon$)
	\STATE Find N corners of $I$, $\{c_1,c_2,\ldots,c_N\}$
	\\ \textit{\%Spatial Dirichlet tessellation}
	\STATE Find $V = \{V_1,V_2,\ldots,V_N \}$ : $d(p,c_i)<d(p,c_j) \forall$ pixels $p\in v_i, i \neq j$
 \\ \textit{\%Feature space Dirichlet tessellation}	
  \FOR {\textbf{each} $V_p \subset V$}		
		\STATE Find set of feature vectors $X_p=\{\overline{x}_1,\overline{x}_2,\ldots,\overline{x}_m\}$ \textit{\% $\overline{x}_i$ = feature vector of pixel $p_i \in V_p$}	 
		\STATE Find the number of clusters $k_p$, cluster centroids $C_p$ and set of clusters of $X_p$ using RSM module
		\STATE $X_p= \{s_1,s_2,\ldots,s_{k_p}\}$ 
 \ENDFOR
\\ \textit{\%Inter Vorono\"{i} region proximal cluster merging}	
	\STATE $S = \{ s_1,s_2,\ldots,s_P\} : S = \bigcup_{\forall X_p}s_p \subset X_p$
	\STATE $minProximity=inf_{\forall s_x,s_y \subset S}\left( d(\overline{\mu_x},\overline{\mu_y}) \right)$
	\WHILE{$minProximity<\varepsilon$}
		\STATE Find the clusters $s_i, s_j\subset S$ that satisfy the merge condition and $\overline{\mu_j})=inf_{\forall s_x,s_y \subset S}\left( d(\overline{\mu_x},\overline{\mu_y}) \right)$
		\STATE $s_k=s_i \cup s_j$
		\STATE Find new cluster centroid
		\STATE P=P-1
		\STATE $minProximity=inf_{\forall s_x,s_y \subset S}\left( d(\overline{\mu_x},\overline{\mu_y}) \right)$
	\ENDWHILE
\STATE $I_{seg} = \left\{ s_1',s_2',\ldots,s_M' \right\}$
\RETURN $I_{seg}$ 
\STATE \hspace{-0.9em} \textbf{end function}
\end{algorithmic} 
\caption{Vorono\"{i}-based image segmentation algorithm}
\end{algorithm}

\subsection{Analysis of the Computational Complexity} \label{sec:compcomplex}
The main components, which contribute to the computational complexity of the AFHA algorithm are AS and FCM algorithms. The computational complexity of AS is $O(ncdi)$ and the computational complexity of FCM is $O(nc^2di)$ \cite{ghosh2013comparative}, where $n$ is the number of data points, $c$ is the number of clusters, $d$ is the number of dimensions, $i$ is the number of iterations. Thus, the computational complexity of AFHA algorithm can be given as $O(ncdi) + O(nc^2di)$. Similarly in the case of RFHA algorithm, the computational complexities of region splitting and merging phases are $O(n^rn^gn^b)$ and $O(idc^2)$ respectively and the complexity of the FCM algorithm is $O(nc^2di)$, where $n^r, n^g \text{ and } n^b$  are the number of valleys in the histograms of R,G and B channels respectively. Thus, the computational complexity of RFHA algorithm is $O(n^rn^gn^b) + O(idc^2) + O(nc^2di)$.

In the proposed algorithm, the regions splitting and merging happens within each Vorono\"{i} region leading to $O(\sum_{k=1}^N n^r_k n^g_k n^b_k)$ and $O(i_kdc_k^2)$ complexities in region splitting and merging stages  respectively, where $N$ is the number of Vorono\"{i} regions. It is important to note that the parameters $n^r_k, n^g_k, n^b_k$ and $c_k$ of the proposed method are much less compared to $n^r, n^g, n^b$ and $c$ of the RFHA algorithm as the region splitting and merging happen within small Vorono\"{i} regions in the proposed method rather than on the whole image as in RFHA. Also, the computational complexity of the K-means algorithm is $O(ncdi)$ \cite{ghosh2013comparative}, which is much less compared to the complexity of FCM $O(nc^2di)$. These factors contribute to lower the computational complexity of the proposed method. The computational complexity of the inter Vorono\"{i} region proximal cluster merging is $O(idc^2)$. Thus, the computational complexity of the proposed method is $O(\sum_{k=1}^N {n^r_k n^g_k n^b_k + i_kdc_k^2}) + O(idc^2) + O(ncdi)$.

The computational complexities of each algorithm reported in this article are summarized in table \ref{tab:tab1}. Computational complexity of Mean-Shift algorithm is given as a reference.

\begin{table}[!htb]
\begin{center}
\begin{tabular}{ l c}
\hline
\textbf{Method} & \textbf{Computational Complexity}\\
\hline
Mean-Shift & $O(n^2di)$\\
AFHA & $O(ncdi) + O(nc^2di)$\\
RFHA & $O(n^rn^gn^b) + O(idc^2) + O(nc^2di)$\\
Proposed & $O(\sum_{k=1}^N {n^r_k n^g_k n^b_k + i_kdc_k^2}) + O(idc^2) + O(ncdi)$\\
\hline
\end{tabular}
\end{center}
\caption{Comparison of computational complexity}
\label{tab:tab1}
\end{table}

$n$ = number of data points, $c$ = number of clusters, $d$ = number of dimensions, $i$ = number of iterations, $N$ = number of Vorono\"{i} regions, $n^r, n^g \text{ and } n^b$ = number of valleys in the histograms of R,G and B channels respectively.

\section{Experimental Results and Analysis} \label{sec:exp}
The proposed algorithm was implemented in Matlab and its results were compared with the results of Matlab implementations of AFHA and RFHA algorithms. For the experiments, the latest version of the Berkeley Segmentation Dataset and Benchmark (BSDS500) \cite{BSDS500,MartinFTM01} was selected. Comparison of results of each algorithm for some sample images from BSDS500 data set is given in figures \ref{fig:fig3-1} and \ref{fig:fig3-2}. During the implementation of the algorithm, $\varepsilon$ is set to 71, which is said to be effective in detecting perceptually near clusters as given in \cite{tan2013color}.

\begin{figure}[!htb]
	\centering
		\includegraphics[width=0.75\textwidth]{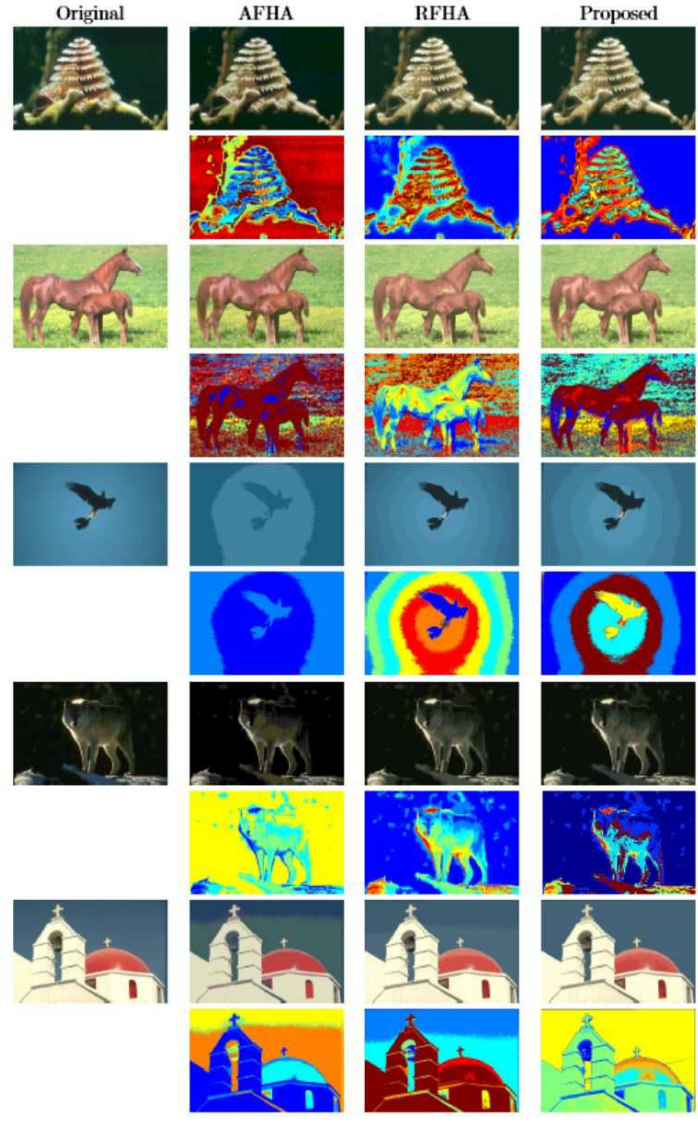}
	\caption{Comparison of results for sample images set 1 from BSD500 data set}
	\label{fig:fig3-1}
\end{figure}

\begin{figure}[!htb]
	\centering
		\includegraphics[width=0.75\textwidth]{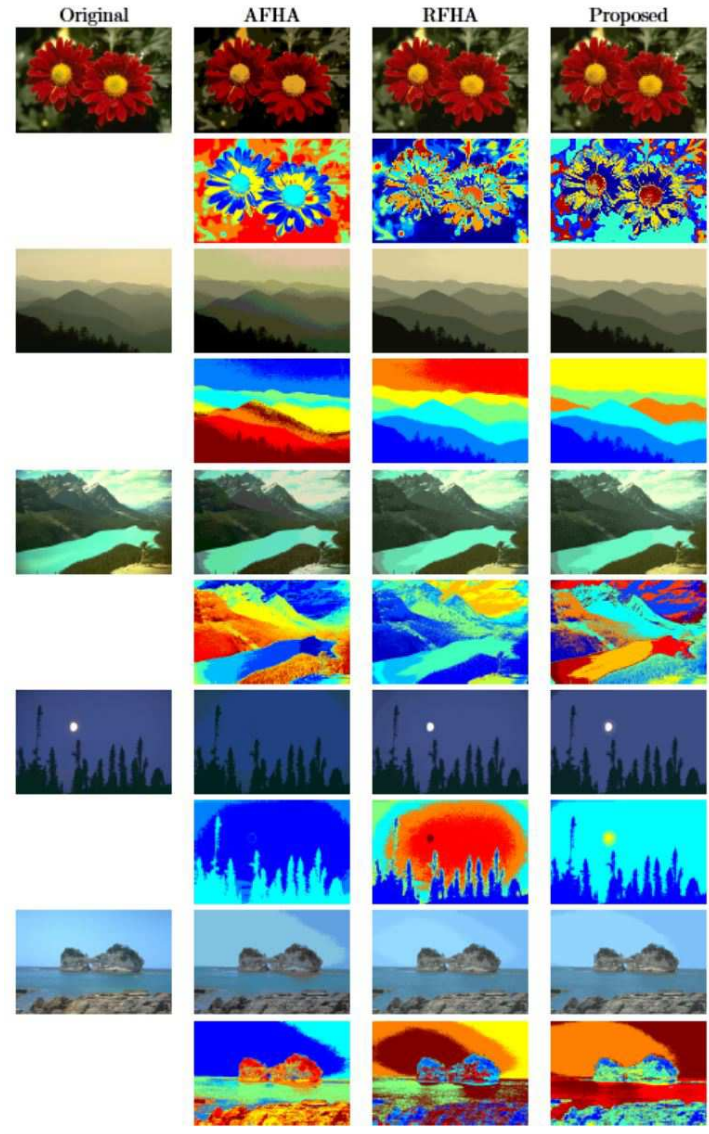}
	\caption{Comparison of results for sample images set 2 from BSD500 data set}
	\label{fig:fig3-2}
\end{figure}

\subsection{Qualitative Analysis of the Segmentation Results}
By observing figures \ref{fig:fig3-1} and \ref{fig:fig3-2}, it is evident that the proposed method outperform the other two methods in terms of the segmentation results for the given sample images. For example, for the Coral image in figure \ref{fig:fig3-1}, both AFHA and RFHA result in non-uniform background (so many isolated pixels in the background) while the proposed method preserves the homogeneity of the segment representing the background. This fact can be clearly seen by zooming in the segments given in false color (row 2 of Coral image results). Also, both AFHA and RFHA results in some misclassified pixels in the background and on the body of the horses in the Horses image. For the Bird image, AFHA fails to give a satisfactory segmentation (centroid colors are far deviated from the original colors) while RFHA produces a segmented image with too many segments in the region of the sky. The proposed method provides a satisfactory segmentation for both these images.

AFHA fails to provide a satisfactory segmentation for the Wolf image as some of the pixels on the Wolf's body are misclassified. The result of RFHA and the proposed method are almost the same except for the fact that RFHA segmented image has more noise in the background. The segmentation result of the proposed method and the RFHA algorithm is very much alike for the River image and Flowers image given in figure \ref{fig:fig3-2}. However, the segmentation of the tree line at the bottom right hand side of the River image segmented by the proposed method is clearer compared to the RFHA segmented image. Furthermore, the proposed method provides the best segmentation result for the Moon image. Both AFHA and RFHA results in over-segmentation of the sky area and AFHA fails to segment the moon in the Moon image. 

Similarly, both AFHA and RFHA results in over-segmentation of the sky region in the Church image and Mountains image and RFHA results in over-segmentation of the sky region in the Sea image. Also, AFHA results in misclassified pixels in the segmented Mountains, Church and Sea images. The proposed method provides a more uniform segmentation results for all three images. 

By observing the results given in figures \ref{fig:fig3-1} and \ref{fig:fig3-2}, it is evident that the proposed method is more robust in segmenting large homogenous regions such as the sky region in all sample images. Segmentation results of AFHA commonly shows misclassified pixels and segments having centroid colors much deviated from the original colors and the segmentation results of RFHA are mostly over-segmented. These facts become evident, when comparing the segmentation results given in false color in figures \ref{fig:fig3-1} and \ref{fig:fig3-2}.

Furthermore, the segmentation results of the proposed method can be further improved by varying the predefined threshold $\varepsilon$ depending on the application. For example, figure \ref{fig:fig4} depicts how the segmentation result of the sample image 2 can be improved by increasing the $\varepsilon$ from 71 to 150. It is important to note that similar change of the predefined threshold ($dc$) in RFHA does not result in such improved segmentation, when the predefined threshold is increased to 150. RFHA results in an under-segmented image for $dc = 150$. 

\begin{figure}[!htbp]
	\centering
		\includegraphics[width=1\textwidth]{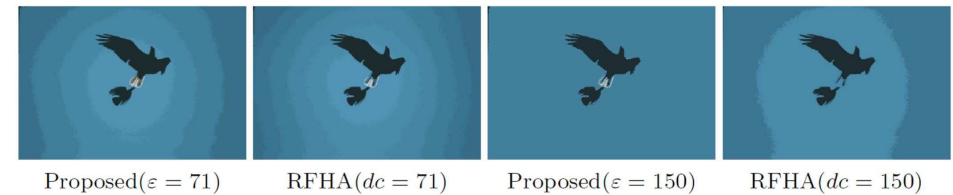}
	\caption{Comparison of results for variations of $\varepsilon$}
	\label{fig:fig4}
\end{figure}

\subsection{Quantitative Analysis of the Segmentation Results} \label{subsec:quant}
There are multiple benchmarks reported in the past literature to evaluate the image segmentation results. Zhang et al. \cite{zhang2008image} broadly categorizes these evaluation methods into supervised evaluation methods, which evaluates segmentation algorithms by comparing the resulting segmented image against a manually segmented reference image or ground truth and unsupervised evaluation methods, which evaluate a segmented image based on how well it matches a broad set of characteristics of segmented images as desired by humans. Supervised evaluation is subjective and time consuming while the unsupervised evaluation is quantitative and objective. In this article, we will be using the unsupervised evaluation methods.

We first started with the \emph{Mean Squared Error (MSE)}, which is one of the most fundamental benchmark used to evaluate cluster quality. MSE can be calculated by using the equation (\ref{eq:eq4}). The number of clusters and cluster quality for the four sample images given in figures \ref{fig:fig3-1} and \ref{fig:fig3-2} are reported in table \ref{tab:tab2}.

\begin{equation}\label{eq:eq4}
	MSE = \frac{1}{N} \sum_{j=1}^M \sum_{i \in S_i} \left\| x_i - c_j \right\|^2
\end{equation}

where $N$ is the total number of pixels in the images, $M$ is the number of clusters produced during the clustering process, $S_j$ is the set of pixels belonging to $j^{th}$ cluster, $c_j$ is the feature vector of the $j^{th}$ cluster centroid and $x_i$ is the feature vector of the $i^{th}$ pixel belonging to $j^{th}$ cluster. Thus, MSE measure the average deviation of the pixels from the cluster centroids. 

\begin{table}[!htbp]
	\centering
		\begin{tabular}{l r r r r r r}
			\hline
			\multirow{2}{*}{Image} & \multicolumn{3}{c}{No. of Clusters} & \multicolumn{3}{c}{MSE (*1.0e+3)}\\			
			 & AFHA & RFHA & Proposed & AFHA & RFHA & Proposed\\
			\hline
			Coral & 9	& 10 & 9	& 1.1727 &	\textbf{0.2886} &	0.2964 \\
			Horses & 23	& 8	& 13	& 1.0874 &	0.6335 &	\textbf{0.3755} \\
			Bird & 2 &	8 &	8 &	1.3135 &	\textbf{0.0479} &	0.0561 \\
			Wolf & 5 &	7 &	7 &	0.8897 &	0.1760 &	\textbf{0.1704} \\
			Church & 6 &	11 &	12 &	1.6880 &	\textbf{0.1567} &	0.1862 \\
			Flowers & 7 &	15 &	18 &	1.9896 &	0.2946 &	\textbf{0.2372} \\
			Mountains & 12 &	7 &	6 &	1.0085 &	\textbf{0.0943} &	0.1156 \\
			River & 16 &	11 &	16 &	1.2827 &	0.4535 &	\textbf{0.2911} \\
			Moon & 3 &	8 &	5 &	1.2717 &	\textbf{0.0413} &	0.0970 \\
			Sea & 8 &	11 &	10 &	1.4449 &	\textbf{0.2283} &	0.2374 \\
			\hline
		\end{tabular}
	\caption{Comparison of no. of clusters and MSE of different algorithms}
	\label{tab:tab2}
\end{table}

The results given in table \ref{tab:tab2} shows that the proposed method results in a very low MSE for all the sample images. For all the sample images, AFHA results in the highest MSE for all sample images. RFHA results in the lowest MSE for six images namely, Coral, Bird, Church, Mountains, Moon and Sea and the proposed method results in the lowest MSE for four images namely, Horses, Wolf, Flowers and River. Overall, for the 200 images that was used in the experiments, the proposed method results in 47\% of the images with the lowest MSE and RFHA results in 53\% of the images with the lowest MSE. Thus, with respect to MSE alone, RFHA seems to perform better than the proposed method.

However, MSE alone has not proven to be a very reliable benchmark for the evaluation of image segmentation results. There is always a trade-off between preserving details and suppressing noise. If there are too many segments in the segmented image, then the difference between pixels belonging to a cluster and the cluster centroid may be lower leading to a smaller MSE. But, since many small clusters are formed and the number of clusters is large, the segmented image may not be a satisfactory one. Therefore, further analysis with regard to the number of clusters and homogeneity of clusters is essential for successful evaluation of the proposed segmentation algorithm.  

Liu and Yang proposed an image segmentation evaluation function $F(I)$ in \cite{liu1994multiresolution}, which penalizes the over-segmentation in segmented images. $F(I)$ can be calculated by using equation (\ref{eq:eq5}).

\begin{equation}\label{eq:eq5} 
	F(I) = \frac{\sqrt{M}}{1000 \times N} \sum_{j=1}^M \frac{e_j^2}{\sqrt{N_j}}  
\end{equation}

where $I$ is the image to be segmented, $M$ is the number of segments in the segmented image, $N_j$ is the number of pixels in the $j^{th}$ segment and $e_j$ is the color error of region $j$. $e_j$ is defined as the sum of the Euclidean distances of the feature vectors between the original image and the segmented image of each pixel region. 

%


The term $\sqrt{M}$ in $F(I)$ penalizes the segmentation which form too many segments. $e_j$ indicates whether or not a region is assigned an appropriate feature (color). If the resulting image is over-segmented, the color error of each segment may be smaller, but since the number of segments is large, the value of $F$ will be large indicating that the segmentation result is not good. On the other hand, if the resulting image is under-segmented, then the number of segments will be reduced, but the color error of each segment will be large leading to a large $F$.

A modified version of $F(I)$ named $F'(I)$ was proposed by Borsotti et al. in \cite{borsotti1998quantitative}. Borsotti et al. propose to modify the global penalization measure $\sqrt{M}$ used in $F(I)$ to make it more robust for noisy images. Borsotti et al. also proposed another evaluation function named $Q(I)$ in \cite{borsotti1998quantitative}, which is said to be more sensitive to small segmentation differences. $Q(I)$ uses a stronger penalization factor to penalize non-homogeneous regions. The equations to calculate $F'(I)$ and $Q(I)$ are given in equations (\ref{eq:eq6}) and (\ref{eq:eq7}).

\begin{equation}\label{eq:eq6} 
	F'(I) = \frac{\sum_{j=1}^M e_j^2 \sqrt{\sum_{a=1}^{MaxArea}[S(a)]^{1+(1/a)}}}{(1000 \times N)\sqrt{N_j}}  
\end{equation}

\begin{equation}\label{eq:eq7} 
	Q(I) = \frac{1}{1000 \times N} \sqrt{M}\sum_{j=1}^M \left[ \frac{e_j^2}{1+\log N_j}+ \left( \frac{S(N_j)}{N_j} \right)^2 \right]
\end{equation}

where $N$ is the total number of pixels in image $I$, $M$ is the number of segments, $N_j$ is the number of pixels in $j^{th}$ segment, $S(a)$ denotes the number of regions in image $I$ that has an area of exactly $a$ and $MaxArea$ denotes the largest region in the segmented image.

Thus, next $F(I)$, $F'(I)$ and $Q(I)$ were measured in order to perform a more comprehensive evaluation of the proposed method. Since, the number of small segments in the segmented images produced by the algorithms reported in this article was very much low, the experimental results for $F(I)$ and $F'(I)$ were the same. As the number of small regions reduces, $\sqrt{\sum_{a=1}^{MaxArea}[S(a)]^{1+(1/a)}}$ approximates to $\sqrt{M}$. Therefore, only the values of $F'(I)$ and $Q(I)$ will be reported in this section. Table \ref{tab:tab3} provides a comparison of the results for $F'(I)$ and $Q(I)$ evaluation functions for different algorithms.

\begin{table}[!htbp]
	\centering
		\begin{tabular}{l r r r r r r}
			\hline
			\multirow{2}{*}{Image} & \multicolumn{3}{c}{$F'(I)$} & \multicolumn{3}{c}{$Q(I)$}\\			
			 & AFHA & RFHA & Proposed & AFHA & RFHA & Proposed \\
			\hline
			Coral & 0.0280 &	0.0101 &	\textbf{0.0093} &	0.3329 &	0.0891 &	\textbf{0.0875} \\
			Horses & 0.0628 &	0.0136 &	\textbf{0.0135} &	0.5503 &	0.1657 &	\textbf{0.1332} \\
			Bird & 0.0067 &	\textbf{0.0016} &	0.0020 &	0.1514 &	\textbf{0.0144} &	0.0200 \\
			Wolf & 0.0140 &	0.0052 &	\textbf{0.0048} &	0.1775 &	0.0454 &	\textbf{0.0443} \\
			Church & 0.0246 &	\textbf{0.0065} &	0.0067 &	0.3926 &	\textbf{0.0532} &	0.0668 \\
			Flowers & 0.0370 &	0.0148 &	\textbf{0.0142} &	0.4788 &	0.1124 &	\textbf{0.1077} \\
			Mountains & 0.0276 &	\textbf{0.0017} &	0.0018 &	0.3730 &	\textbf{0.0229} &	0.0256 \\
			River & 0.0537 &	0.0143 &	\textbf{0.0131} &	0.5149 &	0.1442 &	\textbf{0.1175} \\
			Moon & 0.0086 &	\textbf{0.0011} &	0.0015 &	0.1982 &	\textbf{0.0113} &	0.0246 \\
			Sea & 0.0274 &	0.0072 &	\textbf{0.0071} &	0.3988 &	0.0726 &	\textbf{0.0721} \\
			\hline
		\end{tabular}
	\caption{Comparison of $F'(I)$ and $Q(I)$ evaluation functions of different algorithms}
	\label{tab:tab3}
\end{table}

By observing the results given in table \ref{tab:tab3}, it is evident that the proposed method results in the lowest $F'(I)$ and $Q(I)$ for majority of the sample images. The results of $F'(I)$ and $Q(I)$ are consistent for all sample images. The proposed methods results in the lowest $F'(I)$ and $Q(I)$ for Coral, Horses, Wolf, Flowers, River and Sea images while RFHA results in the lowest $F'(I)$ and $Q(I)$ for Bird, Church, Mountains and Moon images. It is important note that the proposed method had higher MSE compared to RFHA for Coral and Sea images, but the $F'(I)$ and $Q(I)$ values for the same images are lower compared to the RFHA method. This is due to the penalization of over-segmentation available in $F'(I)$ and $Q(I)$. AFHA results in the highest values for $F'(I)$ and $Q(I)$ for all the sample images. Overall, the proposed method results in the lowest $F'(I)$ and $Q(I)$ for 63\% of the 200 images used in experiments. RFHA results in the lowest $F'(I)$ and $Q(I)$ for only 37\% of the 200 images in the data set.

The image segmentation evaluation functions $F$ proposed in \cite{liu1994multiresolution} and $F'$ and $Q$ proposed in \cite{borsotti1998quantitative} fall under the first evaluation criteria given in \cite{haralick1985image}, which measure the intra-cluster uniformity. It is said in \cite{zhang2008image} that $F$, $F'$ and $Q$ are biased towards under-segmentation because they use a weighting factor to penalize against over-segmentation. Also, $F$, $F'$ and $Q$ do not measure the inter-region disparity (the second criteria in \cite{haralick1985image}), which is vital for fair evaluation of segmentation results. 

Thus, next we use the evaluation function $F_{RC}$ proposed by Rosenberger and Chehdi in \cite{rosenberger2000genetic} to cover both evaluation criteria. $F_{RC}$ has two evaluation functions to measure both the intra-region uniformity and inter-region disparity. The first function $\underline{D}(I^j)$ measures the global intra-region uniformity, which quantifies the homogeneity of each region in the segmented image $I^j$ and the second function $\overline{D}(I^j)$ measures the global inter-region disparity between regions.

\begin{equation}\label{eq:eq8} 
	\underline{D}(I^j)=\frac{1}{M}\sum_{j=1}^M \frac{N_j}{N}\underline{D}(R_j)
\end{equation} 

\begin{equation}\label{eq:eq9} 
	\overline{D}(I^j)=\frac{1}{M}\sum_{j=1}^M \frac{N_j}{N}\overline{D}(R_j)
\end{equation} 

$M$ is the number of segments, $N_j$ is the number of pixels in $j^{th}$ segment $R_j$, $N$ is the total number of pixels in image $I_j$. According to \cite{zhang2008image}, the $\underline{D}(R_j)$ in the case of color images is computed as the average squared color error of region $R_j$. The inter-region disparity between two regions is calculated as:

\begin{equation}\label{eq:eq10} 
	D(R_i,R_j)=\frac{\left| E(R_i) - E(R_j) \right|}{NG}
\end{equation} 

where $E(R_i)$ is the average gray-level in the region $R_i$ and $NG$ is the number of gray levels in the image. In the case color images, we will be using the average color difference between the cluster centers to measure $D(R_i,R_j)$.

The intra-region and inter-region metrics were combined in order to find the $F_{RC}$ in \cite{rosenberger2000genetic} as follows.

\begin{equation}\label{eq:eq11} 
	F_{RC}=F(\underline{D}(I^j),\overline{D}(I^j))=\frac{\overline{D}(I^j)-\underline{D}(I^j)}{2}
\end{equation}

\begin{table}[!htbp]
	\centering
		\resizebox{\columnwidth}{!}{%
		\begin{tabular}{l r r r r r r r r r}
			\hline
			\multirow{2}{*}{Image} & \multicolumn{3}{c}{$\underline{D}(I^j)$} & \multicolumn{3}{c}{$\overline{D}(I^j)$} & \multicolumn{3}{c}{$F_{RC}$}\\			
			 & AFHA & RFHA & Proposed & AFHA & RFHA & Proposed & AFHA & RFHA & Proposed \\
			\hline
			Coral & 19.3304 &	\textbf{2.2217} &	3.9567 &	28.0100 &	28.0764 &	\textbf{32.7834} &	4.3398 &	12.9273 &	\textbf{14.4133} \\
			Horses & 2.6996 &	9.5929 &	\textbf{2.3926} &	6.2871 &	\textbf{19.1335} &	12.0146 &	1.7938 &	4.7703 &	\textbf{4.8110} \\
			Bird & 323.2662 & \textbf{0.6789} &	1.3421 &	\textbf{48.0000} & 14.0926 &	16.4910 &	-137.6331 &	6.7069 &	\textbf{7.5744} \\
			Wolf & 66.8444 &	\textbf{2.6871} &	3.6571 &	\textbf{50.9932} &	36.4489 &	39.0588 &	-7.9256 &	16.8809 &	\textbf{17.7008} \\
			Church & 72.7586 &	\textbf{1.4397} &	3.8182 &	\textbf{40.7715} &	23.9536 &	21.5476 &	-15.9936 &	\textbf{11.2569} &	8.8647 \\
			Flowers & 43.7067 &	1.2160 &	\textbf{0.8677} &	\textbf{24.7662} &	12.3263 &	11.9862 &	-9.4702 &	5.55515 &	\textbf{5.55925} \\
			Mountains & 11.0262 &	\textbf{1.8963} &	3.6470 &	19.4786 &	37.0116 &	\textbf{43.2205} &	4.2262 &	17.5576 &	\textbf{19.7868} \\
			River & 5.2477 &	3.5570 &	\textbf{1.2531} &	15.3580 &	\textbf{21.6964} &	18.7902 &	5.0552 &	\textbf{9.0697} &	8.76855 \\
			Moon & 218.7566 &	\textbf{0.8104} &	8.5069 &	35.4797 &	20.7144 &	\textbf{45.2829} &	-91.6385 &	9.9520 &	\textbf{18.3880} \\
			Sea & 37.7216 &	\textbf{1.8485} &	2.5008 &	\textbf{23.8559} &	17.4988 &	19.6872 &	-6.9329 &	7.8252 &	\textbf{8.5932} \\
			\hline
		\end{tabular}
		}
	\caption{Comparison of $\underline{D}(I^j)$, $\overline{D}(I^j)$ and $F_{RC}$ evaluation functions of different algorithms}
	\label{tab:tab4}
\end{table}

A comparison of the values for $\underline{D}(I^j)$, $\overline{D}(I^j)$ and $F_{RC}$ evaluation functions of the proposed method and AFHA and RFHA algorithms is given in table \ref{tab:tab4}. It is said in \cite{zhang2008image} that $F_{RC}$ is more balanced with respect to under-segmentation and over-segmentation with only slight or negligible biases one way or the other. This fact becomes evident by observing the results given in table \ref{tab:tab4}. In the results given in \ref{tab:tab4}, RFHA produces the lowest intra-region uniformity for majority of the sample images and AFHA produces the highest inter-region disparity for the majority of the sample images. However, it is important to note that the proposed method produces the best values for the combined metric $F_{RC}$ for majority of the sample images, which means the proposed method produced a balanced result that preserves both the intra-region uniformity and inter-region disparity at the same time. 

Overall, the proposed method results in the best $F_{RC}$ for 71\% of the 200 images used in experiments. RFHA results in the best $F_{RC}$ for 29\% of the 200 images in the data set while AFHA results in the worst $F_{RC}$ for all 200 images. Thus, we can conclude that the proposed method outperforms both AFHA and RFHA in terms of the image segmentation quality.

\subsection{Analysis of the Execution Time}
Some experiments were conducted in order to compare the execution times of the proposed method with AFHA and RFHA. Table \ref{tab:tab5} shows the execution time of each algorithm run on a Intel Xeon E3-1280 V2 @ 3.60 GHz processor for the 10 sample images reported in the experiments section. Also, the average execution time per image of each algorithm was measured by segmenting the 200 training images in the BSDS500 data set. The results are given in table \ref{tab:tab6}. 

\begin{table}[!htbp]
	\centering
		\begin{tabular}{l c r r r}
			\hline
			\multirow{2}{*}{Image} & \multirow{2}{*}{Image Size} & \multicolumn{3}{c}{Execution Time (seconds)} \\			
			 & & AFHA & RFHA & Proposed\\
			\hline
			Coral & $481 \times 321$ & 38.1181 &	509.9087 &	\textbf{30.0898} \\
			Horses & $481 \times 321$ & 60.8689 &	\textbf{14.7592} &	32.8365 \\
			Bird & $481 \times 321$ & 31.1224 &	47.4568 &	\textbf{3.7746} \\
			Wolf & $481 \times 321$ & 37.9469 &	64.9920 &	\textbf{11.6247} \\
			Church & $481 \times 321$ & 35.5790 &	27.9902 &	\textbf{9.6971} \\
			Flowers & $481 \times 321$ & 45.5225 & 92.5119 & \textbf{21.3567} \\
			Mountains & $481 \times 321$ & 39.5727 & \textbf{6.1321} & 7.1426 \\
			River & $481 \times 321$ & 73.1063 & 40.8520 & \textbf{20.6950} \\
			Moon & $481 \times 321$ & 28.7676 &	58.6831 &	\textbf{9.2689} \\
			Sea & $481 \times 321$ & 46.7545 & 27.3802 & \textbf{21.9857} \\
			\hline
		\end{tabular}
	\caption{Comparison of execution times of different algorithms for sample images}
	\label{tab:tab5}
\end{table}

\begin{table}[!htbp]
\begin{center}
\begin{tabular}{ l c}
\hline
\textbf{Method} & \textbf{Avg. Execution Time (seconds)}\\
\hline
AFHA & 65.3978\\
RFHA & 52.3573\\
Proposed & \textbf{21.8905}\\
\hline
\end{tabular}
\end{center}
\caption{Comparison of average execution time per image for 200 images of BSDS500 data set}
\label{tab:tab6}
\end{table}

By observing the results given in tables \ref{tab:tab5} and \ref{tab:tab6}, we can conclude that the proposed method outperforms AFHA and RFHA in terms of the computational complexity. The average execution time per image of the proposed method is roughly 22 seconds, which is the lowest compared to the other two methods. The average execution time per image of the AFHA method is roughly 65 seconds, which is the highest and the average execution time of the RFHA method is roughly 52 seconds. Thus, the average execution time per image of the proposed method is improved by 57.69\% compared to the RFHA method. Thus, the proposed method is proven to be more suitable for real-time image segmentation applications.

\section{Conclusion} \label{sec:conc}

In this article, a new adaptive unsupervised algorithm based on Vorono\"{i} regions was proposed to solve the image segmentation problem. The proposed algorithm is capable of automatically determining the number of clusters and the cluster centroids in a given set of pixels. The proposed algorithm adaptively divides the image into Vorono\"{i} regions and performs region splitting and merging within Vorono\"{i} regions to find intra-Vorono\"{i} region clusters, which will then be iteratively merged to find the final number of clusters and cluster centroids. In contrast to existing algorithms, K-means clustering algorithm is used to find the final segmented image in the proposed method in place of the FCM algorithm. The Vorono\"{i} region wise clustering in the proposed algorithm leads to significant reduction in the computational complexity of the image segmentation problem. Furthermore, since the number of possible clusters within a single Vorono\"{i} region is usually lower compared to the number of clusters in the whole image, estimating the number of clusters and cluster centroids becomes more efficient and precise. 

The experimental results reported in this article confirm that the proposed method outperforms two other adaptive unsupervised cluster-based image segmentation algorithms, AFHA and RFHA in terms of the image segmentation quality based on three different unsupervised image segmentation evaluation benchmarks. Also, the results of the experiments on average execution time per image prove that the proposed method is much faster compared to the other two algorithms reported in this article, which makes the proposed method more suitable for real-time image segmentation applications.

\bibliographystyle{amsplain}
\bibliography{mybibfile}

\end{document}